\documentclass[lettersize,journal]{IEEEtran}
\usepackage[authoryear]{natbib}
\usepackage[utf8]{inputenc} 
\usepackage[T1]{fontenc}    
\usepackage{hyperref}       
\usepackage{url}            
\usepackage{booktabs}       
\usepackage{amsfonts}       
\usepackage{nicefrac}       
\usepackage{microtype}      
\usepackage{subcaption} 
\usepackage{graphicx}
\usepackage[linesnumbered,ruled,vlined]{algorithm2e} 
\usepackage{amsmath}   
\usepackage{amssymb}   
\usepackage{booktabs}
\usepackage{multirow}
\usepackage{caption}
\usepackage{makecell}
\usepackage{graphicx}
\usepackage{amsmath}
\usepackage{geometry}
\usepackage{amsthm}

\newtheorem{theorem}{Theorem}[section]

\theoremstyle{remark}
\newtheorem{remark}{Remark}[section]

\usepackage{booktabs, tabularx, siunitx}

\SetKwComment{Comment}{/* }{ */ \hfill}
\SetCommentSty{commentstyle}
\SetKwInput{Require}{Require}
\SetKwInput{Ensure}{Ensure}
\usepackage{siunitx}
\usepackage{etoolbox}
\usepackage{xcolor}         
\begin{document}

\title{EP-GRPO: Entropy-Progress Aligned Group Relative Policy Optimization with Implicit Process Guidance}

\author{
	\IEEEauthorblockN{Song Yu\textsuperscript{1}, Li Li\textsuperscript{1*}, Wenwen Zhao\textsuperscript{1}, Zhisheng Yang\textsuperscript{1}}\\
	\IEEEauthorblockA{
		\textsuperscript{1}\textit{School of Computer and Information Science, Southwest University, Chongqing 400715, China}\\
		yusong0929@email.swu.edu.cn\\
		lily@swu.edu.cn (Corresponding author.)\\
		zhaoww71@email.swu.edu.cn\\
		yzsheng@email.swu.edu.cn
	}
}

\maketitle
\begingroup\renewcommand\thefootnote{\textsuperscript{*}}
\footnotetext{Corresponding author.}
\endgroup

\begin{abstract}
	Reinforcement learning with verifiable rewards (RLVR), particularly Group Relative Policy Optimization (GRPO), has advanced LLM reasoning. However, GRPO suffers from three credit assignment failures: uniform token-level granularity that ignores heterogeneous informational value, uniform polarity that penalizes correct steps and rewards incorrect ones, and zero-variance collapse that erases outcome-driven gradients. We systematically quantify these failures, revealing highly non-uniform token informativeness, widespread step-level polarity misalignment, and substantial training waste.
	To address these limitations, we propose Entropy-Progress Aligned GRPO (EP-GRPO), a framework that mines the model's intrinsic information flow for dense, self-supervised guidance. EP-GRPO integrates entropy-gated modulation to prioritize high entropy decision pivots, implicit process signals from policy divergence anchored to outcome advantages for directional token-level feedback without external reward models, and cumulative entropy mapping that enables progress-aligned advantage normalization, naturally maintaining gradient flow under zero reward variance.
	Extensive experiments on mathematical reasoning benchmarks demonstrate that EP-GRPO achieves superior accuracy and efficiency compared to GRPO and its variants. The code will be available.
\end{abstract}

\section{Introduction}
\begin{figure*}[t]
	\centering
	\includegraphics[width=\textwidth]{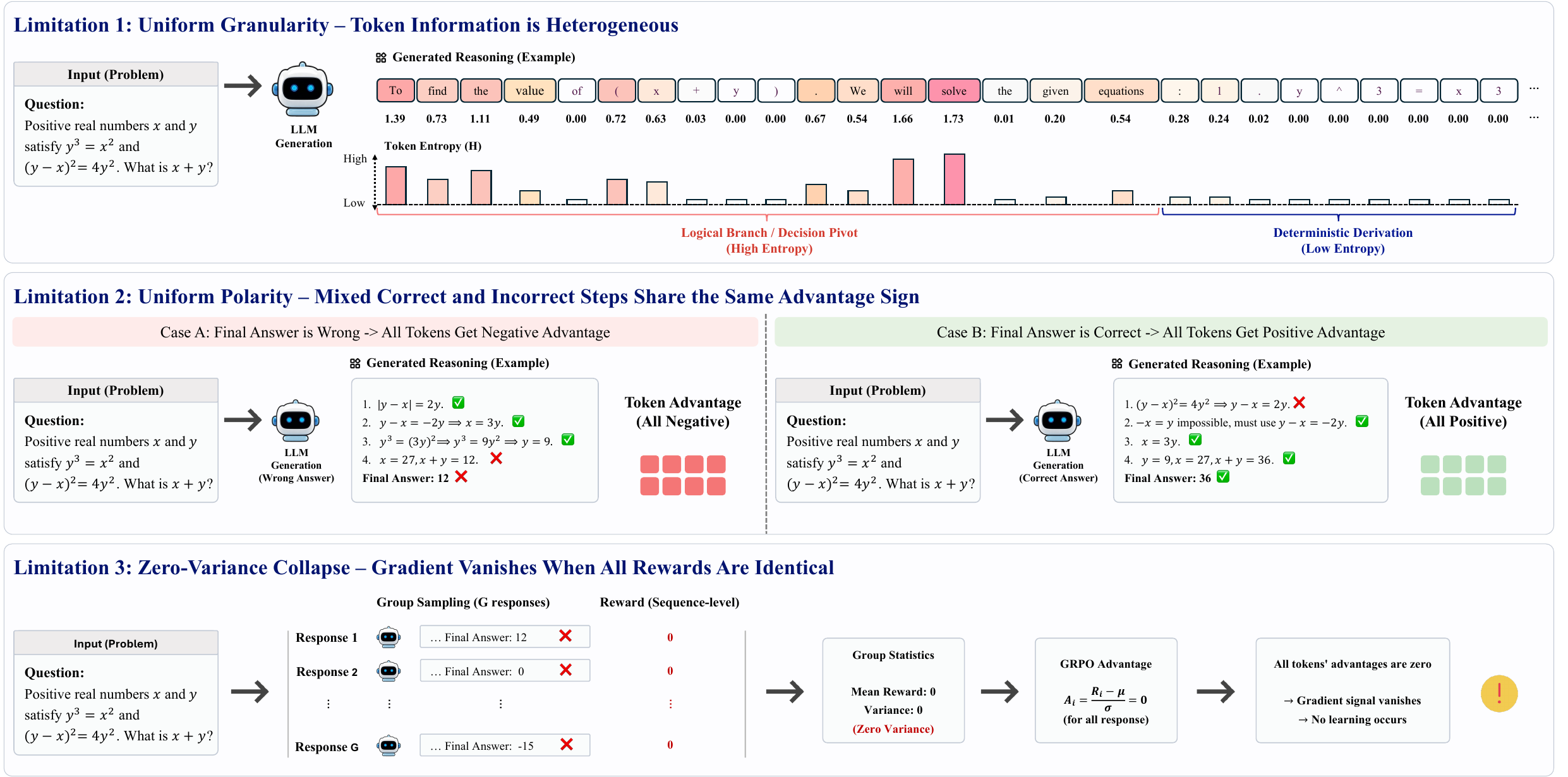}
\caption{\textbf{Conceptual illustration of the fundamental limitations in standard GRPO.} The top panel demonstrates \textbf{Uniform Granularity}, where the model fails to distinguish between critical high entropy decision pivots and deterministic low entropy derivations. The middle panel shows \textbf{Uniform Polarity}, where sequence-level rewards lead to the indiscriminate reinforcement or penalization of both correct and incorrect intermediate steps. The bottom panel illustrates \textbf{Zero-Variance Collapse}, where identical rewards within a group cause the learning signal to vanish.}
	\label{fig:intro_overview}
\end{figure*}

\IEEEPARstart{W}{ith} the advancement of large language models (LLMs), reinforcement learning with verifiable rewards (RLVR) has become the cornerstone for enhancing complex reasoning. 
The release of DeepSeek-R1 \citep{guo2025deepseek} marks the significant effectiveness of RLVR \citep{lamberttulu} in enhancing model reasoning abilities, powered by Group Relative Policy Optimization (GRPO) \citep{shao2024deepseekmath}. By employing intra-group relative advantages, GRPO improves training efficiency without a critic network. 

Despite its success, GRPO still has limitations, as shown in Figure \ref{fig:intro_overview}. First, it assumes uniform token-level advantage contribution across the entire sequence \citep{shao2024deepseekmath,guo2025deepseek}. This coarse granularity ignores the inherent heterogeneity of reasoning processes. At logical branch points, token distributions exhibit high entropy that reflects active exploration, while at deterministic derivation steps, low entropy tokens often encode stylized redundancy rather than genuine reasoning progress \citep{wangbeyond}. Recent efforts such as Process Reward Models (PRMs) \citep{lightman2023letsverifystepstep,wang2024mathshepherd} provide step-level supervision but require costly human annotations and generalize poorly across domains. Entropy-based signals \citep{kadavath2022languagemodelsmostlyknow} and length-aware penalties \citep{singhal2023longwaygo} operate at the trajectory or position level, failing to differentiate individual tokens by their actual information content. 

Second, GRPO assigns the same sign of advantage to all tokens within a sequence. Since the sequence-level reward determines whether advantage is positive or negative, every token in a successful response receives positive reinforcement, while every token in a failed response receives negative reinforcement. This uniform polarity ignores a fundamental reality: a reasoning chain invariably contains a mixture of correct and incorrect steps. Under GRPO, every one of these correct steps is indiscriminately penalized alongside genuine errors, suppressing valid reasoning patterns. Recent works such as PRPO \citep{ding2026prpo} and Step-GRPO  \citep{li2026step} employ a pretrained PRM to score individual steps and aligns process reward distributions with sequence-level outcome advantages. However, these approach still depends on an external reward model. 

Third, GRPO's advantage normalization collapses when all sampled responses within a group receive identical rewards. In this case, the normalized advantage becomes identically zero for all tokens, resulting in a complete loss of the outcome-driven gradient signal. While increasing the group size mitigates this issue, it proportionally escalates computational cost, making it impractical for resource-constrained settings. Recent efforts have addressed this from complementary angles: data-level methods such as DAPO \citep{yu2025dapo} detect and filter zero-variance groups, while signal-level approaches such as RL-ZVP \citep{le2025no} and Dr. GRPO \citep{dr_grpo} reshape advantages via entropy shaping or learned bias terms. Orthogonal work including RE-GRPO \citep{liu2025re} employs LLM-guided reflection to repair failed trajectories, at the cost of additional inference overhead. These methods either rely on extra sampling, auxiliary model calls, or heuristic advantage reshaping that does not explicitly align with logical reasoning progress.

To solve these problems, we propose \textbf{E}ntropy-\textbf{P}rogress Aligned \textbf{GRPO} (EP-GRPO), a progress-aware optimization algorithm that transforms traditionally sparse, outcome-only feedback into dense, token-level learning signals. EP-GRPO addresses three orthogonal limitations of GRPO through complementary mechanisms. First, to resolve the granularity issue, it applies entropy-gated modulation to the original sequence-level advantage. Second, to resolve the polarity issue, it leverages policy divergence as a token-level signal and anchors its polarity to the outcome advantage, yielding directional process rewards without external models. Third, to resolve the zero-variance issue, when group rewards are identical and the outcome advantage vanishes, the anchoring signal degrades to the raw reward polarity. Since the implicit signal depends on policy divergence rather than reward variance, it remains fully active and sustains gradient flow even when outcome comparisons fail. The implicit advantages are then normalized within cumulative entropy buckets to enable progress-aligned, dense rewards, and combined to form the final token-level learning signal.

Our main contributions are as follows:
\begin{itemize}
	\item We introduce an entropy-gated modulation mechanism that transforms the coarse-grained sequence-level advantage into token-level weights, prioritizing gradient updates at critical decision pivots.
	\item We propose a method to convert policy divergence into directional process rewards by anchoring its polarity with the outcome advantage, providing fine-grained token-level feedback without external reward models.
	\item We design a logical-progress alignment mechanism based on cumulative entropy mapping, which projects asynchronous reasoning trajectories into a unified coordinate system. It not only enables progress-aligned normalization of process rewards, but also ensures that the implicit process advantage provides meaningful gradient signals even when outcome-based advantages vanish due to zero reward variance.
\end{itemize}

\section{Related Work}

\subsection{Token-Level Granularity in Credit Assignment}

The coarse granularity of sequence-level rewards has been a longstanding limitation in RLVR. Recent work has begun to address this by moving toward token-level or step-level credit assignment. GTPO \citep{tan2025gtpo} proposes a true token-level reward by dynamically weighting each token with its predictive entropy, providing fine-grained credit to high entropy decision points. Similarly, research \citep{liu2025fromuniform} explicitly tailor policy optimization to each token's nature based on entropy, treating heterogeneous tokens with differentiated optimization weight rather than uniform updates. From a causal perspective, \citep{khandoga2026beyond} introduce counterfactual importance weighting, which masks reasoning spans and measures the resulting drop in answer probability to identify tokens with genuine causal impact, effectively upweighting critical tokens while suppressing filler tokens.

Exploration enhancing strategies provide another angle on token-level differentiation. SIREN \citep{jiang2025rethinking} selectively applies entropy regularization to specific logical tokens, attempting to focus exploration where it matters most. Research \citep{cheng2026reasoning} directly augment the GRPO advantage with a clipped, gradient-detached per token entropy term, encouraging exploration at pivotal tokens. In a complementary line, SEED-GRPO \citep{chen2025seed} employs semantic entropy rather than token-level entropy to quantify prompt-level uncertainty, enabling uncertainty-aware conservative updates for difficult prompts.

Despite these advances, existing token-level methods primarily focus on exploration or importance weighting. They do not address the polarity of token-level signals, whether a token’s contribution is positive or negative relative to the final outcome. Moreover, they rely on heuristic entropy shaping rather than progress-aligned signal extraction from the policy's own information dynamics.

\subsection{Process-Level Polarity and Directional Feedback}

Beyond granularity, the uniform polarity of sequence-level advantages has motivated research on process-level supervision. Explicit PRMs \citep{lightman2023letsverifystepstep,wang2024mathshepherd} train separate verifiers on human-annotated or synthetically generated reasoning steps, providing dense step-wise directional feedback. Process Advantage Verifiers \citep{setlur2024rewardingprogress} extend this by measuring the change in future correctness probability induced by each step. More recent methods such as PRPO \citep{ding2026prpo} and Step-GRPO \citep{li2026step} employ pretrained PRMs to score individual steps and align process reward distributions with sequence-level outcome advantages.

Despite these advances, explicit PRMs incur high annotation costs and frequently suffer from reward hacking or poor cross-domain generalization. In response, implicit and self-supervised process reward frameworks have emerged. PRIME \citep{cui2025prime} enables online PRM updates using only policy rollouts and outcome labels, mining implicit process signals directly from the policy's predictive dynamics. TEPO \citep{lin2025token} uses Markov likelihood to link group-level rewards to token-level optimization via token-mean aggregation, incorporating entropy control to avoid collapse. However, these implicit mechanisms still operate under trajectory-level or physical-position alignment, leaving the heterogeneous timing of logical progress unaddressed.

Work on length bias and verbosity \citep{singhal2023longwaygo} further highlights that sequence-level or position-based penalties remain insufficient because they do not capture intrinsic information content. This motivates a self-supervised process signal that provides token-level directional guidance aligned with logical progress rather than physical token position.

\subsection{Zero-Variance Collapse and Gradient Continuity}
A critical failure mode of GRPO occurs when all responses in a group receive identical rewards, causing the normalized advantage to vanish and the outcome-driven gradient signal to collapse. This issue is particularly prevalent in early stage training or on highly challenging reasoning tasks, where policies tend to produce uniformly incorrect outputs, or conversely in saturated regimes where all outputs are correct. Existing solutions fall into three categories.

Data-level interventions detect and filter uninformative groups before training. DAPO \citep{yu2025dapo} introduces Dynamic Sampling with decoupled clipping to identify and resample zero-variance groups, ensuring informative gradient updates. GRESO \citep{zheng2025greso} proposes selective rollout that predicts and skips uninformative prompts based on reward training dynamics, reducing wasted computation.

Task-level rebalancing addresses the uneven distribution of zero-variance frequency across different tasks. MT-GRPO \citep{ramesh2026multitaskgrpo} observes that zero-gradient rates vary drastically across tasks in multi-task settings, and combines a Ratio-Preserving Sampler with Improvement-Aware Task Reweighting to dynamically allocate optimization effort, substantially improving worst-case task performance.

Signal-level methods reshape the advantage computation itself to avoid collapse. RL-ZVP \citep{le2025no} constructs non-zero advantages from per-token policy entropy even under zero reward variance, enabling continued learning from otherwise wasted prompts. Dr. GRPO \citep{dr_grpo} introduces a learned bias term that perturbs the advantage calculation to prevent collapse. EDGE-GRPO \citep{zhang2025edge} combines guided error correction with entropy-driven advantage shaping to inject diversity into the advantage distribution without additional sampling. Complementary to these, RE-GRPO \citep{liu2025re} maintains a hard negative case pool and applies LLM-guided reflection with dual validation to repair failed trajectories, which are then distilled back into the policy, effectively converting collapsed groups into useful supervision at the cost of added inference.

Despite these advances, a principled signal source that remains informative under zero reward variance without external computation, and that aligns naturally with logical reasoning progress, remains an open challenge.
\section{Preliminaries}
We begin by formalizing the probabilistic framework of LLMs \citep{brown2020language}, followed by an overview of RLVR. Then, we delineate the GRPO algorithm, providing the necessary background for our proposed method.

\textbf{LLMs.} Specifically, given an input prompt $x$, an LLM $\pi_{\theta}$ sequentially generates a $T$-token response $y=(y_1,...,y_T)$:
\begin{equation}
\pi_\theta(\mathbf{y}|\mathbf{x}) = \prod_{t=1}^T \pi_\theta(y_t|\mathbf{x}, \mathbf{y}_{<t}).
\end{equation}

\textbf{RLVR.}
RLVR \citep{lamberttulu} is a family of reinforcement learning methods that utilize verifiable reward signals rather than learned reward models. Unlike RLHF \citep{ouyang2022training}, RLVR employs rule-based objective rewards, such as the correctness of a programming output, the correctness of the final answer to a mathematical problem, or compliance with formatting. These rewards originate from tasks with explicit ground truth verification.

Consider a dataset $\mathcal{D} = \{(x, y)\}$ where $x$ is the prompt and $y$ is the ground truth. The optimization objective in RLVR is to maximize the expected reward:

\begin{equation}
\mathcal{J}(\theta) = \mathbb{E}_{(x,y) \sim \mathcal{D}} \left[ \mathbb{E}_{\hat{y} \sim \pi_\theta(\cdot|x)} [R(\hat{y}, y)] \right],
\end{equation}

\noindent{where $R(\hat{y}, y)$ is a verifiable reward function that compares the generated output $\hat{y}$ against the ground truth $y$.}

\textbf{GRPO.}
GRPO \citep{shao2024deepseekmath} is a more efficient policy optimization algorithm compared with Proximal Policy Optimization (PPO) \citep{schulman2017proximal}, as it estimates advantages through group-based response sampling, eliminating the need for a separate value network. For each prompt $x$, GRPO samples $G$ responses ${\{o_1, o_2, \dots, o_G\}}$ from the policy model $\pi_{\theta}$, with each response consisting of $|o_i|$ tokens. These responses are evaluated using a reward model or function $R(x, o_i)$, yielding a reward $r_i$ for each response.

Token-level advantages $\hat{A}_{i,t}$ are computed through within-group normalization. Specifically, for all tokens in response $o_i$, the advantage is set to the normalized reward of that response:

\begin{equation}
\hat{A}_{i,t} = \frac{r_i - \text{mean}(\mathbf{r})}{\text{std}(\mathbf{r})+\delta}, \quad \forall t \in \{1,...,|o_i|\},
\end{equation}

\noindent where $\mathbf{r} = [r_1, r_2, ..., r_G]$ is the vector of rewards for all responses in the group, $\text{mean}(\cdot)$ and $\text{std}(\cdot)$ denote the mean and standard deviation operations respectively, and $\delta$ is a small constant for numerical stability. This normalization provides a relative comparison of responses within the same group, effectively estimating advantages. This assignment assumes an equal contribution of all tokens to the final outcome reward, bypassing the need for per-token value estimation.

To prevent the policy from diverging too far from the reference policy, GRPO incorporates KL divergence regularization. Specifically, for each token position, the KL divergence is estimated using a low-variance approximation \citep{schulman2020approxkl}:

\begin{equation}
	\begin{aligned}
		\mathbb{D}_{\text{KL}}[\pi_{\theta} \| \pi_{\text{ref}}] = & \frac{\pi_{\text{ref}}(o_{i,t} \mid x, o_{i,<t})}{\pi_{\theta}(o_{i,t} \mid x, o_{i,<t})} \\
		& - \log \frac{\pi_{\text{ref}}(o_{i,t} \mid x, o_{i,<t})}{\pi_{\theta}(o_{i,t} \mid x, o_{i,<t})} - 1.
	\end{aligned}
\end{equation}

This estimator provides a low-variance approximation of the KL divergence while maintaining computational efficiency.

The GRPO optimization objective combines a clipped surrogate objective with a KL divergence penalty. After each generation, multiple updates can be performed using the following loss function:

\begin{equation}
	\begin{aligned}
		\mathcal{L}_{\text{GRPO}}(\theta&) =  -\frac{1}{\sum_{i=1}^G |o_i|} \sum_{i=1}^G \sum_{t=1}^{|o_i|} \bigg[ \min \Big( \rho_{i,t}(\theta) \hat{A}_{i,t}, \\
		& \text{clip}(\rho_{i,t}(\theta), 1-\epsilon, 1+\epsilon) \hat{A}_{i,t} \Big) - \beta \mathbb{D}_{\text{KL}}[\pi_{\theta} \| \pi_{\text{ref}}] \bigg]
	\end{aligned}
\end{equation}
\noindent where $\rho_{i,t}(\theta) = \frac{\pi_{\theta}(o_{i,t} \mid x, o_{i,<t})}{\pi_{\theta_{\text{old}}}(o_{i,t} \mid x, o_{i,<t})}$ is the probability ratio between the current policy and the old policy, $\text{clip}(\cdot, 1-\epsilon, 1+\epsilon)$ constrains the probability ratio to the interval $[1-\epsilon, 1+\epsilon]$ to prevent excessively large policy updates, $\epsilon$ is the clipping hyperparameter, and $\beta$ controls the strength of KL regularization.

This loss function aims to maximize expected rewards while constraining the magnitude of policy updates, ensuring training stability. Through the combination of group-based advantage estimation and KL regularization, GRPO achieves stable policy optimization without requiring a value network, making it particularly suitable for reinforcement learning tasks with verifiable rewards. 

\section{Motivation}
\label{sec:motivation}
\begin{figure*}[t]
	\centering
	\includegraphics[width=\textwidth]{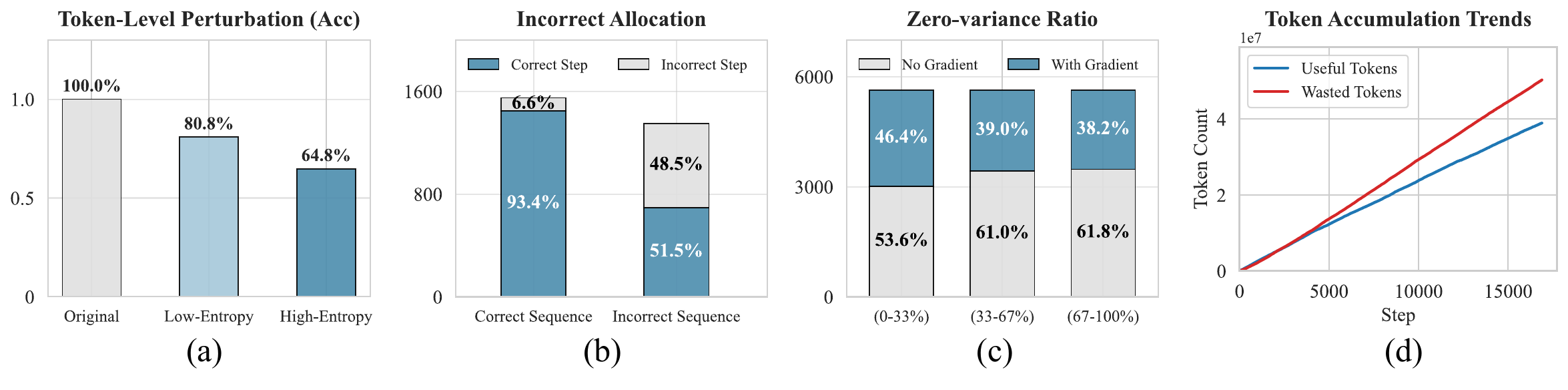}
	\caption{\textbf{Empirical analysis using standard GRPO.}
		(a) Perturbing high entropy tokens degrades sequence accuracy substantially more than perturbing low entropy tokens, highlighting their disproportionate importance.
		(b) Incorrect tokens in otherwise correct sequences frequently receive positive advantages, while correct tokens in failed sequences receive negative ones, indicating poor credit assignment.
		(c) Severe zero-variance issues persist throughout training.
		(d) During the middle-to-late stages, over 50\% of tokens belong to such zero-variance groups, rendering them ineffective for policy updates.}
	\label{fig:four_experiments}
\end{figure*}
In this section, we perform more empirical and theoretical analysis to reveal the underlying mechanisms and provide new insights.

\subsection{Contribution by Token Entropy}
To quantify the impact of token-level heterogeneity, we conducted perturbation experiments on a GRPO-trained policy. From the MATH \citep{hendrycksmath2021} dataset, we selected 2,460 problems that the model answered correctly. For each problem, we identified the 5\% of tokens with the highest predictive entropy and the 5\% with the lowest entropy in the generated reasoning sequence, randomly replaced them with alternative tokens, and measured the resulting change in final answer accuracy. As shown in Figure \ref{fig:four_experiments}(a), perturbing high entropy tokens caused a more than 3.5× larger accuracy drop compared to perturbing low entropy tokens. We hypothesize that this phenomenon occurs because high entropy tokens typically coincide with strategic decision points where the model chooses among competing reasoning paths, so perturbing them derails the entire subsequent trajectory; low entropy tokens, in contrast, largely appear in deterministic derivations where the model can often self-correct using surrounding logical context to preserve coherence.
\subsection{Polarity misalignment in the reasoning process}
To further quantify the polarity misalignment, we sampled 500 problems from the MATH dataset and generated responses using a GRPO-trained policy. To obtain unbiased step-level correctness annotations, each response was automatically segmented into atomic reasoning steps by an external LLM. Every atomic step was independently evaluated for local correctness by the LLM, crucially without access to the final answer outcome, thereby preventing hindsight bias. As shown in \ref{fig:four_experiments}(b), the results reveal a striking misalignment. Among correctly answered problems, 6.6\% of intermediate reasoning steps were locally erroneous yet received uniformly positive advantages under GRPO. More critically, among incorrectly answered problems, 51.5\% of steps were locally correct, yet these were all indiscriminately assigned negative advantages. These results demonstrate that GRPO inadvertently rewards errors in successful trajectories while suppressing valid reasoning in failed ones.
\subsection{Vanishing gradients during training}
To quantify the prevalence of zero-variance collapse, we analyzed training logs from a standard GRPO run on the MATH dataset with G=8. Across 16,882 training steps, we recorded whether the reward standard deviation within each group fell to zero. As shown in Figure \ref{fig:four_experiments}(c), this failure mode is not a rare edge case: 58.77\% of all training steps (9,922 out of 16,882) produced zero-variance groups, rendering over 46 million tokens computationally wasted, as shown in Figure \ref{fig:four_experiments}(d). Furthermore, the problem intensifies as training progresses. Figure \ref{fig:four_experiments}(c) reveals that the zero-variance ratio rises from 53.6\% in the early stage, to 61.0\% in the middle stage, and reaches 61.8\% in the late stage. This upward trend indicates that the model does not naturally outgrow this failure mode; instead, outcome-driven gradients become progressively sparser, leaving an increasing fraction of training steps guided only by the KL regularization term rather than by reward signals. \textbf{These empirical findings raise a natural question: why does GRPO suffer from these limitations?} We now turn to a theoretical analysis of GRPO's gradient dynamics.
\begin{table*}[htbp]
	\centering
	\caption{Confusion Matrix of Sign Combinations.}
	\label{tab:confusion}
	\begin{tabular}{c|c|c}
		\toprule
		& $\hat{A}_i > 0$ (Winner Group) & $\hat{A}_i < 0$ (Loser Group) \\
		\midrule
		$s_{i,t} > 0$ & \textbf{True Positive (TP)} & \textbf{False Positive (FP)} \\
		(More Certain) & Model is more confident in a correct path. & Model is more confident in a wrong path. \\
		& $\rightarrow$ \textit{Strongest Reward} & $\rightarrow$ \textit{Strongest Penalty} \\
		\midrule
		$s_{i,t} < 0$ & \textbf{False Negative (FN)} & \textbf{True Negative (TN)} \\
		(Deviated) & Model deviated from reference on a correct path. & Model deviated from reference on a wrong path. \\
		& $\rightarrow$ \textit{Mild Penalty} & $\rightarrow$ \textit{Implicit Reward} \\
		\bottomrule
	\end{tabular}
\end{table*}
\subsection{Entropy-Driven Gradient Contribution}
Policy optimization fundamentally relies on gradient updates. We therefore begin by analyzing the origins of gradient signals in sequence generation tasks. Consider policy:
\begin{equation} 
	\pi_\theta(y_t \mid x, y_{<t})
\end{equation}
and cross-entropy loss:
\begin{equation} 
	\mathcal{L}_t = -\log \pi_\theta(y_t \mid x, y_{<t})
\end{equation}

Taking the partial derivative of $\mathcal{L}_t$ with respect to the model parameters $\theta$ via the chain rule:

\begin{equation}
	\frac{\partial \mathcal{L}_t}{\partial \theta} = \frac{\partial \mathcal{L}_t}{\partial \pi_\theta} \cdot \frac{\partial \pi_\theta}{\partial \theta} = \left(-\frac{1}{\pi_\theta}\right) \cdot \frac{\partial \pi_\theta}{\partial \theta}
\end{equation}

The gradient intensity is defined as the magnitude of the contribution from the loss derivative with respect to $\pi_\theta$:

\begin{equation}
	\left\| \frac{\partial \mathcal{L}_t}{\partial \pi_\theta} \right\| = \frac{1}{\pi_\theta}
\end{equation}

This derivation reveals a fundamental fact: the magnitude of the gradient signal at each token position is determined solely by the model's prediction probability $\pi_\theta$, independent of any specific reinforcement learning algorithm. When the model is uncertain about a token, which means a high entropy token, the loss function automatically assigns a larger gradient to drive learning.

Having established the source of gradient intensity, we now examine how GRPO interacts with this inherent gradient signal. The loss function can be simplified to:

\begin{equation}
	\mathcal{L}_{\text{GRPO}} = -\frac{1}{G} \sum_{i=1}^{G} \sum_{t=1}^{T} \hat{A}_i \cdot \log \pi_\theta(y_{i,t} \mid x_i, y_{i,<t})
\end{equation}

where $\hat{A}_i$ is the advantage value for sequence $i$, computed via group-based reward normalization. Critically, $\hat{A}_i$ is a sequence-level constant, it takes the same value for all tokens $t$ within the same sequence.

Taking the partial derivative with respect to $\theta$ via the chain rule:

\begin{equation}
	\frac{\partial \mathcal{L}_{\text{GRPO}}}{\partial \theta} = -\hat{A}_i \cdot \frac{1}{\pi_\theta} \cdot \frac{\partial \pi_\theta}{\partial \theta}
\end{equation}

We can decompose this gradient as:

\begin{equation}
	\underbrace{\frac{\partial \mathcal{L}_{\text{GRPO}}}{\partial \theta}}_{\text{GRPO Gradient}} = \underbrace{\hat{A}_i}_{\text{GRPO's Contribution}} \times \underbrace{\left( -\frac{1}{\pi_\theta} \cdot \frac{\partial \pi_\theta}{\partial \theta} \right)}_{\text{Natural Gradient}}
\end{equation}

We can see that GRPO does not actively create or amplify the gradient disparity between high entropy and low entropy tokens. Instead, it passively inherits the natural gradient $1/\pi_\theta$ from the cross-entropy loss, and simply scales the entire sequence uniformly by $\hat{A}_i$. It makes no distinction between tokens based on their information content, treating critical reasoning pivots and routine deterministic tokens identically. This passive inheritance, while computationally efficient, leaves substantial room for improvement by actively modulating token-level contributions based on their informational value.

\subsection{Dense Rewards and Gradient Reversal}
\label{sec:dense_rewards}
In GRPO, all tokens within a sequence share the identical sign of $\hat{A}_i$. Correct tokens in a failed trajectory are uniformly penalized, while incorrect tokens in a successful trajectory are inadvertently rewarded. The ideal solution is token-level directional feedback—dense rewards that assign positive advantage to correct tokens and negative advantage to incorrect ones.

Existing PRMs provide step-level supervision but incur substantial annotation costs and generalize poorly across domains. A free process signal is needed. Consider parameterizing an outcome reward as the log-likelihood ratio between the current policy and a reference model: $r_\theta(y) = \beta \log \frac{\pi_\theta(y)}{\pi_{\text{ref}}(y)}$. As shown in prior work~\citep{cui2025prime}, training with this reward implicitly learns a Q-function of the form:
\begin{equation}
	\begin{aligned}
		q_\theta^t(y_{<t}, y_t) &= \sum_{i=1}^t \beta \log \frac{\pi_\theta(y_i \mid y_{<i})}{\pi_{\text{ref}}(y_i \mid y_{<i})} \\
		&= \beta \log \mathbb{E}_{\pi_{\text{ref}}(y \mid y_{<t})} [e^{r_\theta(y)}].
	\end{aligned}
\end{equation}
From this, the per-step process reward emerges by taking the difference of adjacent Q-functions:
\begin{equation}
	\begin{aligned}
		r_t^p &= q_\theta^t - q_\theta^{t-1} = \beta \log \frac{\pi_\theta(y_t \mid y_{<t})}{\pi_{\text{ref}}(y_t \mid y_{<t})} \\
		&\quad + \beta \log \frac{\pi_\theta(y_{t-1} \mid y_{<t-1})}{\pi_{\text{ref}}(y_{t-1} \mid y_{<t-1})}.
	\end{aligned}
\end{equation}

Thus, the token-level log-probability divergence $s_{i,t} = \log \pi_\theta(o_{i,t} \mid x, o_{i,<t}) - \log \pi_{\text{ref}}(o_{i,t} \mid x, o_{i,<t})$ constitutes a principled, self-supervised process signal requiring no external annotations. Intuitively:
\begin{itemize}[nosep]
	\item $s_{i,t} > 0$: The model is more confident than the reference model.
	\item $s_{i,t} < 0$: The model is less certain than the reference model.
\end{itemize}

However, $s_{i,t}$ lacks intrinsic directionality. Increased confidence is beneficial only when converging toward a correct solution, but harmful when reinforcing an error. The sequence-level outcome advantage $\hat{A}_i$ provides the missing directional context. Table~\ref{tab:confusion} summarizes the four interpretable regimes arising from the interaction of $\text{sign}(\hat{A}_i)$ and $\text{sign}(s_{i,t})$, yielding a symbolically anchored dense reward without external supervision. Taken together, these analyses point to a clear direction: amplify token-level advantages at high entropy decision pivots while suppressing redundant low entropy derivations; leverage the self-supervised process signal $s_{i,t}$ for directional token-level feedback; and exploit its independence from reward variance to naturally resolve zero-variance collapse. These three principles form the foundation of our proposed method.
\section{Method}
\label{sec:method}
\begin{figure*}[t]
	\centering
	\includegraphics[width=\textwidth]{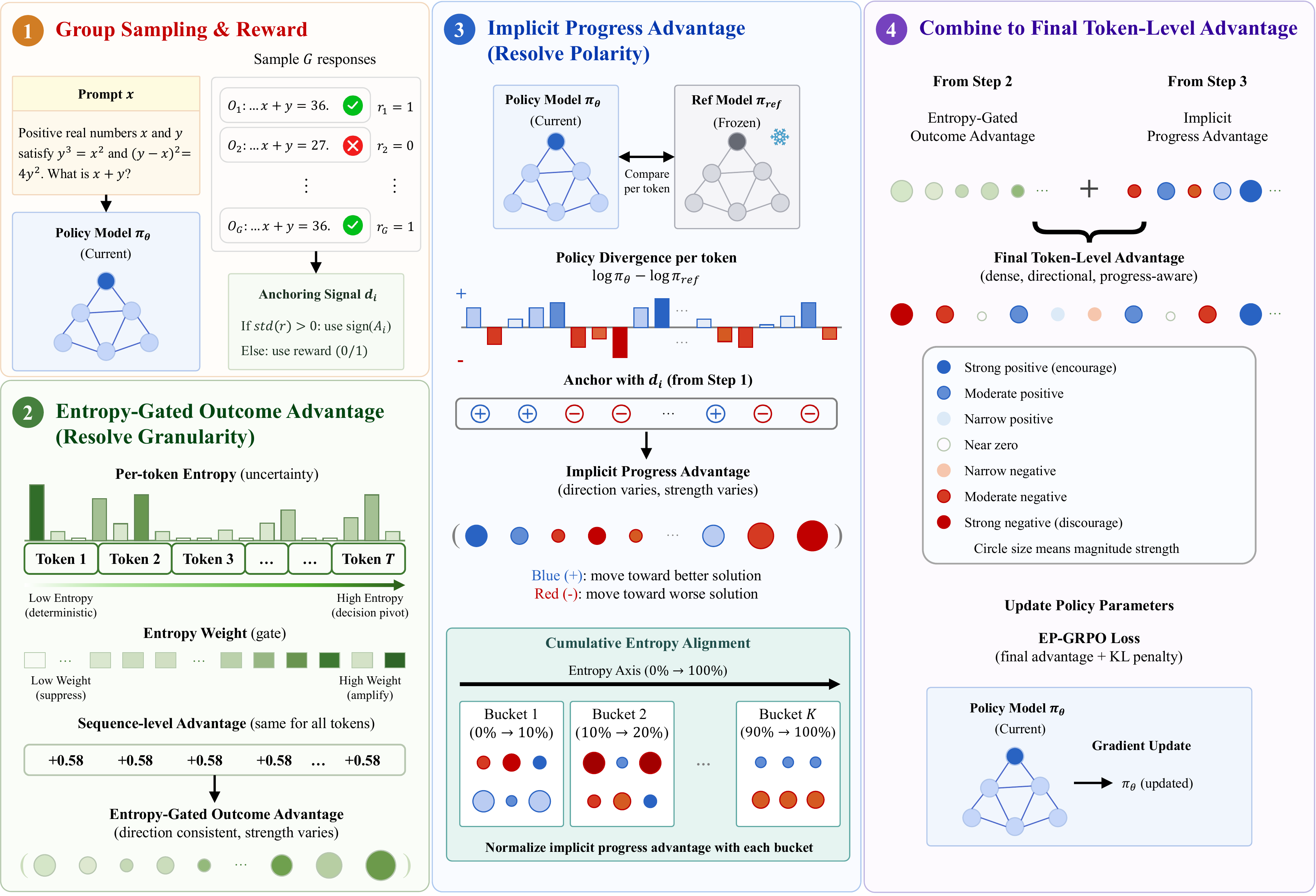}
	\caption{\textbf{Overview of EP-GRPO.}
		EP-GRPO transforms sparse outcome rewards into dense, token-level learning signals. It first applies entropy-gated modulation to assign higher weights to important tokens while keeping a consistent direction. Then, an implicit progress signal is computed by comparing the policy with a reference model, providing token-level feedback with varying magnitude and direction, further aligned via cumulative entropy-based bucketing. Finally, both signals are combined into a unified token-level advantage for stable policy optimization.}
	\label{fig:method_overview}
\end{figure*}
Building on the design principles established in Section~\ref{sec:motivation}, we now introduce EP-GRPO, as shown in Figure \ref{fig:method_overview}. Our method involves three components: entropy-gated outcome advantage modulation, an implicit progress advantage derived from policy divergence, and their combination under a unified token-level objective.

\subsection{Entropy-Gated Outcome Advantage}
Given an input prompt $x$, the policy $\pi_\theta$ generates a response $o_i$ of length $|o_i|$. For each token $o_{i,t}$, we compute the predictive entropy:
\begin{equation}
	H_{i,t} = -\sum_{v \in \mathcal{V}} \pi_\theta(v \mid x, o_{i,<t}) \log \pi_\theta(v \mid x, o_{i,<t}),
\end{equation}
where $\mathcal{V}$ denotes the vocabulary. High $H_{i,t}$ indicates uncertainty at logical branch points, while low $H_{i,t}$ corresponds to deterministic derivations.

To differentiate high entropy decision pivots from low entropy redundancy, we introduce an entropy-gated modulation. For a group of $G$ responses, we first compute the group-level mean and standard deviation of entropy across all tokens:
\begin{equation}
	\mu_H = \frac{1}{\sum_{i=1}^G |o_i|} \sum_{i=1}^G \sum_{t=1}^{|o_i|} H_{i,t}.
\end{equation}
\begin{equation}
	\sigma_H = \sqrt{\frac{1}{\sum_{i=1}^G |o_i|} \sum_{i=1}^G \sum_{t=1}^{|o_i|} (H_{i,t} - \mu_H)^2}.
\end{equation}
Each token's entropy is then standardized and mapped through a sigmoid gate to produce a weight $W_{i,t}$:
\begin{equation}
	W_{i,t} = \sigma\!\left( \gamma \cdot \frac{H_{i,t} - \mu_H}{\sigma_H + \epsilon} \right),
\end{equation}
where $\sigma(\cdot)$ is the sigmoid function, $\gamma > 0$ controls the steepness of the gate, $\epsilon$ ensures numerical stability.

The entropy-gated outcome advantage is then:
\begin{equation}
	\hat{A}_{i,t}^{\text{outcome}} = W_{i,t} \cdot \hat{A}_i.
\end{equation}
This modulation preserves the original sign of $\hat{A}_i$ while adjusting magnitude to reflect token-level informational content, amplifying gradients at critical decision points and suppressing redundant derivations.

\subsection{Implicit Progress Advantage}
The entropy-gated outcome advantage still inherits the uniform polarity of $\hat{A}_i$. To provide token-level directional feedback, we leverage the policy divergence against a frozen reference model $\pi_{\text{ref}}$:
\begin{equation}
	s_{i,t} = \lambda\left(\log \pi_\theta(o_{i,t} \mid x, o_{i,<t}) - \log \pi_{\text{ref}}(o_{i,t} \mid x, o_{i,<t})\right),
\end{equation}
where $\lambda > 0$ scales the sensitivity. As derived in Section~\ref{sec:dense_rewards}, $s_{i,t}$ constitutes a self-supervised process signal requiring no external annotations.

Since $s_{i,t}$ intrinsically lacks directionality, it must be anchored to an external sign. We define the anchoring signal $d_i$ as:
\begin{equation}
	d_i = \begin{cases}
		\text{sign}(\hat{A}_i), & \text{if } \text{std}(\mathbf{r}) > 0, \\
		\text{sign}(r_i - \theta_{\text{reward}}), & \text{if } \text{std}(\mathbf{r}) = 0,
	\end{cases}
\end{equation}
where $\theta_{\text{reward}}$ is a reward threshold (typically 0.5 for binary rewards). When group reward variance is non-zero, the outcome advantage provides directional context as in Table~\ref{tab:confusion}. When variance collapses, $\hat{A}_i$ vanishes and the anchor gracefully degrades to the raw reward, ensuring $s_{i,t}$ always receives meaningful polarity. The anchored signal is:
\begin{equation}
	\tilde{s}_{i,t} = d_i \cdot s_{i,t}.
\end{equation}

A direct global normalization of $\tilde{s}_{i,t}$ would be statistically unsound, as the distribution of $s_{i,t}$ varies systematically with reasoning phase. We introduce a progress-aligend normalization strategy named \textbf{\textit{logical progress buckets}} defined by cumulative entropy. For each response, we compute the cumulative entropy and map it to a relative information progress $\tau_{i,t} \in [0, 1]$:
\begin{equation}
	S_{i,t} = \sum_{k=1}^t H_{i,k}, \quad \tau_{i,t} = S_{i,t} \,/\, S_{i,|o_i|}.
\end{equation}
Tokens from all group responses are partitioned into $K$ equal-width buckets based on $\tau_{i,t}$. Within each bucket $\mathcal{B}_k$, we compute the Z-score of $\tilde{s}_{i,t}$:
\begin{equation}
	\tilde{s}_{i,t}^{\text{norm}} = \frac{\tilde{s}_{i,t} - \mu_k}{\sigma_k + \epsilon}, \quad \forall (i,t) \in \mathcal{B}_k,
\end{equation}
where $\mu_k$ and $\sigma_k$ are the bucket mean and standard deviation. This aligns tokens by logical advancement rather than physical position: tokens at similar reasoning stages are compared against each other regardless of their absolute sequence index.

The implicit progress advantage is then:
\begin{equation}
	\hat{A}_{i,t}^{\text{progress}} = \eta \cdot \tilde{s}_{i,t}^{\text{norm}},
\end{equation}
where $\eta > 0$ controls its contribution relative to the outcome advantage.
\subsection{Final Token-Level Advantage}
The complete token-level advantage combines both components:
\begin{equation}
	\hat{A}_{i,t}^{\text{final}} = \hat{A}_{i,t}^{\text{outcome}} + \hat{A}_{i,t}^{\text{progress}}.
\end{equation}

When $\text{std}(\mathbf{r}) > 0$, both components are active: entropy-gated modulation resolves the granularity issue, and the anchored implicit signal provides token-level polarity. When $\text{std}(\mathbf{r}) = 0$, the outcome advantage $\hat{A}_{i,t}^{\text{outcome}}$ vanishes, but $\hat{A}_{i,t}^{\text{progress}}$ remains fully operative. Its anchor degrades to the raw reward $r_i$, and its magnitude depends solely on policy divergence. The zero-variance collapse is thus resolved without additional sampling, auxiliary models, or heuristic advantage shaping.

\section{Theoretical Analysis}
\label{sec:analysis}

We establish a theoretical properties of EP-GRPO. We prove that its gradient is equivalent to that of a regularized GRPO objective, providing a principled interpretation of the optimization dynamics. 

\subsection{Gradient Equivalence and Implicit Regularization}

We begin by establishing the relationship between the EP-GRPO gradient and the original GRPO objective. For analytical simplicity, we consider the typical case where $\text{std}(\mathbf{r}) > 0$.

\begin{theorem}
	\label{thm:gradient_equivalence}
	The EP-GRPO gradient is equivalent to the gradient of the GRPO objective augmented with an entropy-weighted squared log-ratio regularization term. Formally, there exists a potential function $F(\pi_\theta)$ such that:
	\begin{equation}
		\nabla_\theta \mathcal{J}_{\text{EP-GRPO}}(\theta) = \nabla_\theta \mathcal{J}_{\text{GRPO}}(\theta) + \eta \cdot \nabla_\theta F(\pi_\theta),
	\end{equation}
	where $F(\pi_\theta) = \frac{\beta}{2} \mathbb{E}_{i,t}\left[ \Lambda_{i,t} \left( \log \frac{\pi_\theta(o_{i,t})}{\pi_{\text{ref}}(o_{i,t})} \right)^2 \right]$, and $\Lambda_{i,t}$ captures the entropy-gated modulation and sign anchoring.
\end{theorem}

\begin{proof}
	Recall from Section~\ref{sec:method} that the EP-GRPO token-level advantage decomposes as:
	\begin{equation}
		\hat{A}_{i,t}^{\text{final}} = \hat{A}_{i,t}^{\text{outcome}} + \hat{A}_{i,t}^{\text{progress}} = W_{i,t} \hat{A}_i + \eta \cdot \tilde{s}_{i,t}^{\text{norm}}.
	\end{equation}
	
	The policy gradient for EP-GRPO is therefore:
	\begin{equation}
		\nabla_\theta \mathcal{J}_{\text{EP-GRPO}} = \mathbb{E}\left[ \sum_{i,t} \nabla_\theta \log \pi_\theta(o_{i,t}) \cdot \hat{A}_{i,t}^{\text{final}} \right].
	\end{equation}
	
	Substituting the decomposition and isolating the progress advantage term yields:
	\begin{equation}
		\nabla_\theta \mathcal{J}_{\text{EP-GRPO}} = \nabla_\theta \mathcal{J}_{\text{GRPO}}^{\text{gated}} + \eta \cdot \mathbb{E}\left[ \sum_{i,t} \nabla_\theta \log \pi_\theta(\cdot) \cdot \tilde{s}_{i,t}^{\text{norm}} \right],
	\end{equation}
	where $\mathcal{J}_{\text{GRPO}}^{\text{gated}}$ denotes the GRPO objective with entropy-gated weighting.
	
	Under the assumption $\text{std}(\mathbf{r}) > 0$, the anchored signal is $\tilde{s}_{i,t} = \text{sign}(\hat{A}_i) \cdot \beta (\log \pi_\theta - \log \pi_{\text{ref}})$. During gradient computation, bucket-level normalization statistics $\mu_k$, $\sigma_k$ are treated as constants computed from the old policy $\pi_{\theta_{\text{old}}}$. Let $\Lambda_{i,t}$ denote the combined scaling factor absorbing $\text{sign}(\hat{A}_i)$, the normalization constant, and $\eta$. The progress gradient term becomes:
	\begin{equation}
		\begin{aligned}
			\Delta \nabla_\theta = \beta \cdot \mathbb{E}\biggl[ \sum_{i,t} \Lambda_{i,t} \cdot& \nabla_\theta \log \pi_\theta(o_{i,t}) \biggl( \log \pi_\theta(o_{i,t})\\
			& - \log \pi_{\text{ref}}(o_{i,t}) \biggr) \biggr].
		\end{aligned}
	\end{equation}
	
	Using the identity $\nabla_\theta \log \pi \cdot \log \pi = \frac{1}{2} \nabla_\theta (\log \pi)^2$, we obtain:
	\begin{equation}
		\begin{aligned}
			\Delta \nabla_\theta = \beta \cdot \mathbb{E}&\Bigg[ \sum_{i,t} \Lambda_{i,t} \cdot \nabla_\theta \Big( \frac{1}{2} (\log \pi_\theta(o_{i,t}))^2 \\
			&- \log \pi_{\text{ref}}(o_{i,t}) \log \pi_\theta(o_{i,t}) \Big) \Bigg].
		\end{aligned}
	\end{equation}
	
	Completing the square, this is exactly the gradient of:
	\begin{equation}
		F(\pi_\theta) = \frac{\beta}{2} \mathbb{E}_{i,t}\left[ \Lambda_{i,t} \left( \log \frac{\pi_\theta(o_{i,t})}{\pi_{\text{ref}}(o_{i,t})} \right)^2 \right],
	\end{equation}
	which completes the proof.
\end{proof}

\begin{remark}
	Theorem~\ref{thm:gradient_equivalence} reveals that EP-GRPO implicitly penalizes large deviations from the reference policy, with penalty strength modulated by $\Lambda_{i,t}$. Crucially, $\Lambda_{i,t}$ inherits the entropy-gated weight $W_{i,t}$, meaning tokens at critical decision points (high entropy) receive stronger regularization than those in deterministic derivations (low entropy). This transforms the passive, uniform KL constraint of standard GRPO into an active, precision-guided mechanism that stabilizes exploration precisely where it matters most.
\end{remark}

\section{Experiments}
\subsection{Experimental Setup}

\begin{table*}
	\centering
	\small 
	\rmfamily
	\caption{\textnormal{Performance comparison across reasoning benchmarks. Acc denotes sample accuracy over all responses and Fmt denotes format correctness (boxed rate). Sampling temperature is set to 1.0, and the maximum output length is limited to 2048 tokens. Each problem is sampled 16 times to reduce variance. Best results are in bold and second-best are underlined.}}
	\label{tab:qwen_results_full}
	\renewcommand{\arraystretch}{1.2}
	\setlength{\tabcolsep}{3.5pt}
	
	\begin{tabular}{l cc cc cc cc cc | cc}
		\toprule
		& \multicolumn{2}{c}{\textbf{Math 500}}
		& \multicolumn{2}{c}{\textbf{AMC 23}} 
		& \multicolumn{2}{c}{\textbf{Minerva}} 
		& \multicolumn{2}{c}{\textbf{AIME 24}} 
		& \multicolumn{2}{c}{\textbf{AIME 25}} 
		& \multicolumn{2}{c}{\textbf{Avg}} \\
		\cmidrule(lr){2-3} \cmidrule(lr){4-5} \cmidrule(lr){6-7} \cmidrule(lr){8-9} \cmidrule(lr){10-11} \cmidrule(lr){12-13}
		\textbf{Model} 
		& Acc & Fmt
		& Acc & Fmt
		& Acc & Fmt
		& Acc & Fmt
		& Acc & Fmt
		& Acc & Fmt \\
		\midrule
		
		\textit{Commercial Models} \\
		
		DeepSeek-R1-671B-0528
		& 47.03 & 58.28
		& 33.91 & 33.91
		& 11.41 & 34.38
		& 13.54 & 13.54
		& 11.04 & 11.04
		& 23.39 & 30.23 \\
		
		Qwen3-235B-A22b-Instruct
		& 72.03 & 87.19
		& 47.81 & 55.00
		& 17.66 & 92.34
		& 24.58 & 29.58
		& 16.88 & 20.21
		& 35.79 & 38.40 \\
		
		\midrule
		
		\textit{Qwen2.5-3B} \\
		
		Base
		& 31.56 & 73.12
		& 15.94 & 74.38
		& 3.13 & 65.94
		& 2.08 & 79.38
		& 0.83 & 80.00
		& 10.71 & 74.56 \\
		
		GRPO
		& 50.94 & 98.28
		& 28.28 & 96.56
		& \underline{6.25} & 99.22
		& \underline{4.58} & 91.88
		& 0.63 & 95.21
		& 18.14 & 96.23 \\
		
		- Higher Temp
		& 51.41 & 97.50
		& \underline{32.19} & 96.09
		& 5.78 & 98.12
		& 4.17 & 92.71
		& 1.04 & 97.08
		& 18.92 & 96.30 \\
		
		- More Rollouts
		& \underline{55.16} & 99.22
		& 30.31 & 96.25
		& 5.78 & 98.28
		& 2.92 & 92.71
		& \underline{1.46} & 97.50
		& \underline{19.13} & 96.79 \\
		
		EP-GRPO
		& \textbf{56.88} & 99.06
		& \textbf{39.53} & 95.31
		& \textbf{9.06} & 96.88
		& \textbf{6.25} & 90.00
		& \textbf{2.92} & 97.50
		& \textbf{22.93} \textbf{\textit{(+26.4\%)}} & 95.75 \\
		
		\midrule
		
		\textit{Qwen2.5-7B} \\
		
		Base
		& 29.06 & 79.53
		& 23.13 & 82.34
		& 3.13 & 77.03
		& 3.96 & 80.42
		& 1.88 & 84.58
		& 12.23 & 80.78 \\
		
		GRPO
		& \underline{61.25} & 99.22
		& \underline{47.19} & 96.88
		& \underline{13.13} & 98.59
		& \underline{8.75} & 94.17
		& \underline{5.21} & 96.04
		& \underline{27.11} & 96.98 \\
		
		EP-GRPO
		& \textbf{64.53} & 99.53
		& \textbf{49.38} & 97.03
		& \textbf{15.47} & 99.06
		& \textbf{12.92} & 94.38
		& \textbf{9.38} & 96.67
		& \textbf{30.34} \textbf{\textit{(+11.9\%)}} & 97.33 \\
		
		\bottomrule
	\end{tabular}
\end{table*}

\textbf{Datasets and Evaluation Benchmarks.} 
For reinforcement learning, we used the Skywork-OR1-RL-Data \citep{he2025skywork}, a curated dataset of 105K math problems with verified answers. From this dataset, we randomly sampled 8,000 problems for training. To evaluate generalization and reasoning capabilities, we tested on five competitive mathematics benchmarks: MATH500 \citep{hendrycksmath2021}, AMC23 \citep{amc23}, Minerva \citep{lewkowycz2022solving}, AIME24 \citep{aime24} and AIME25 \citep{aime25}. All benchmarks were evaluated using accuracy, format rate and pass@$k$ metrics.

\textbf{Model Configurations.} 
We evaluated EP-GRPO on two scales of the Qwen2.5 series \citep{qwen2025qwen25}: 3B and 7B. To ensure training efficiency, we employed Low-Rank Adaptation (LoRA) \citep{hu2022lora} with rank $r=32$ and $\alpha_{\text{lora}}=64$, targeting all linear layers in both the attention and feed-forward modules. For EP-GRPO, we set the implicit signal scale $\lambda=0.1$, progress advantage weight $\eta=0.2$, number of cumulative entropy buckets $K=10$, entropy gate steepness $\gamma=5.0$ and reward threshold $\theta_{\text{reward}}=0.5$. 

\textbf{Training Specifications.} 
All models were trained for 1,000 steps using the TRL framework \citep{vonwerra2020trl} with a maximum sequence length of 2,048 tokens. We used a group size of $G=8$ rollouts per prompt with temperature $T=1.0$ and top-$p=0.95$ for sampling. The learning rate was set to $5\times 10^{-6}$ with a 0.1 warmup ratio and linear decay, optimized using AdamW with weight decay $0.001$. The $\beta$ of KL divergence is set to 0.001. Training employed a global batch size of 16. All experiments were conducted on NVIDIA RTX PRO 6000 Blackwell 96GB GPUs with a fixed random seed 42.

\subsection{Baselines}
To comprehensively evaluate EP-GRPO, we compare against the following configurations:

\begin{itemize}
	\item \textbf{Base Model.} The pre-trained Qwen2.5 model without any fine-tuning, serving as the performance lower bound.
	
	\item \textbf{Commerical Model.} A high-parameter model that can be accessed through a user interface has been launched.
	
	\item \textbf{GRPO.} The standard Group Relative Policy Optimization \citep{shao2024deepseekmath} under identical rollout and reward settings. 
	
	\item \textbf{Higer Temp.} Increases the sampling temperature from $1.0$ to $1.2$ while keeping $G=8$.
	\item \textbf{More Rollouts.} Increases the rollout from $G=8$ to $G=10$ while keeping $T=1.0$.
	
	\item \textbf{Ablation Configurations.} To assess the contribution of each EP-GRPO component, we evaluate four stripped-down variants on Qwen2.5-3B, each isolating a specific mechanism:
	\begin{itemize}
		\item \textbf{+EG}: GRPO with entropy gating only, scaling sequence advantage to token-level advantage.
		\item \textbf{+IPS}: GRPO with implicit progress signals and cumulative entropy bucketing.
		\item \textbf{+EG+IPS}: GRPO with entropy gating and implicit progress signals (including bucketing), but without zero-variance degradation.
		\item \textbf{+IPS+ZVD}: GRPO with implicit progress signals (including bucketing) and zero-variance degradation, without entropy gating.
	\end{itemize}
\end{itemize}
\begin{table*}[t]
	\centering
	\caption{\textnormal{Comprehensive $pass@k$ evaluation ($k \in \{4,8,16\}$).}}
	\label{tab:qwen_pass_k_updated}
	\renewcommand{\arraystretch}{1.2}
	\setlength{\tabcolsep}{2.2pt}
	
	\resizebox{\textwidth}{!}{ 
		\rmfamily
		\begin{tabular}{l | ccc | ccc | ccc | ccc | ccc}
			\toprule
			& \multicolumn{3}{c|}{\textbf{Math 500}}
			& \multicolumn{3}{c|}{\textbf{AMC 23}} 
			& \multicolumn{3}{c|}{\textbf{Minerva}} 
			& \multicolumn{3}{c|}{\textbf{AIME 24}} 
			& \multicolumn{3}{c}{\textbf{AIME 25}} \\
			\textbf{Model} 
			& @4 & @8 & @16 
			& @4 & @8 & @16 
			& @4 & @8 & @16 
			& @4 & @8 & @16 
			& @4 & @8 & @16 \\
			\midrule
			
			\textit{Commercial Models} \\
			
			DeepSeek-R1-671B-0528
			& 66.63 & 69.33 & 70.00
			& 58.72 & 66.19 & 72.50
			& 20.94 & 23.78 & 25.00
			& 28.07 & 32.14 & 33.33
			& 23.68 & 28.82 & 33.33 \\
			
			Qwen3-235B-A22B-Instruct
			& 76.06 & 78.07 & 80.00
			& 59.66 & 63.00 & 65.00
			& 21.44 & 23.40 & 25.00
			& 28.32 & 30.00 & 33.33
			& 22.21 & 24.22 & 26.67 \\
			
			\midrule
			
			\textit{Qwen2.5-3B} \\
			
			Base
			& 78.16 & 95.28 & 99.79
			& 50.15 & 75.27 & 94.00
			& 11.95 & 22.54 & 40.20
			& 8.10 & 15.61 & 28.99
			& 3.30 & 6.52 & 12.72 \\
			
			GRPO
			& 94.26 & 99.68 & \textbf{100.00}
			& 73.64 & 93.12 & 99.55
			& \underline{22.80} & \underline{40.50} & \underline{64.84}
			& \underline{17.16} & \underline{31.49} & \underline{53.37}
			& 2.48 & 4.93 & 9.69 \\
			
			- Higher Temp
			& 94.48 & 99.70 & \textbf{100.00}
			& \underline{78.95} & \underline{95.62} & \underline{99.82}
			& 21.24 & 38.07 & 61.88
			& 15.70 & 29.04 & 49.95
			& 4.11 & 8.09 & 15.65 \\
			
			- More Rollouts
			& \underline{96.00} & \underline{99.85} & \textbf{100.00}
			& 76.51 & 94.54 & 99.72
			& 21.24 & 38.07 & 61.88
			& 11.20 & 21.23 & 38.20
			& \underline{5.72} & \underline{11.17} & \underline{21.25} \\
			
			EP-GRPO
			& \textbf{96.58} & \textbf{99.89} & \textbf{100.00}
			& \textbf{86.71} & \textbf{98.26} & \textbf{99.97}
			& \textbf{31.68} & \textbf{53.44} & \textbf{78.54}
			& \textbf{22.82} & \textbf{40.56} & \textbf{64.99}
			& \textbf{11.20} & \textbf{21.23} & \textbf{38.20} \\
			
			\midrule
			
			\textit{Qwen2.5-7B} \\
			
			Base
			& 74.78 & 93.70 & 99.62
			& 65.17 & 87.96 & 98.60
			& 11.95 & 22.54 & 40.20
			& 14.96 & 27.79 & 48.15
			& 7.31 & 14.15 & 26.49 \\
			
			GRPO
			& \underline{97.78} & \underline{99.95} & \textbf{100.00}
			& \underline{92.29} & \underline{99.42} & \textbf{100.00}
			& \underline{43.12} & \underline{67.77} & \underline{89.77}
			& \underline{30.75} & \underline{52.20} & \underline{77.45}
			& \underline{19.32} & \underline{35.02} & \underline{58.10} \\
			
			EP-GRPO
			& \textbf{98.44} & \textbf{99.98} & \textbf{100.00}
			& \textbf{93.49} & \textbf{99.59} & \textbf{100.00}
			& \textbf{49.03} & \textbf{74.14} & \textbf{93.44}
			& \textbf{42.60} & \textbf{67.21} & \textbf{89.47}
			& \textbf{32.64} & \textbf{54.78} & \textbf{79.84} \\
			
			\bottomrule
		\end{tabular}
	}
\end{table*}

\subsection{Main Results}
To validate the effectiveness of EP-GRPO across different model scales, we compare EP-GRPO against all baselines on Qwen2.5-3B and Qwen2.5-7B.

Table~\ref{tab:qwen_results_full} reports the sample accuracy (Acc) and format correctness (Fmt) of all methods across five mathematical reasoning benchmarks. EP-GRPO consistently outperforms standard GRPO at both model scales. At the 3B scale, EP-GRPO achieves the highest accuracy on all five datasets, improving the average accuracy from GRPO's 18.14\% to 22.93\%. At the 7B scale, EP-GRPO also achieves the highest accuracy on five datasets, with the average accuracy increasing from 27.11\% to 30.34\%. Notably, at the 3B scale, EP-GRPO even surpasses DeepSeek-R1-671B on AMC23 (39.53\% vs.\ 33.91\%) and approaches it on Minerva (9.06\% vs.\ 11.41\%), demonstrating the substantial potential of improved credit assignment in RLVR under limited parameter budgets.

Comparisons with two GRPO variants further rule out alternative explanations. Higher Temp and More Rollouts fail to consistently outperform standard GRPO across most datasets and are significantly inferior to EP-GRPO in all cases. This confirms that the performance gains of EP-GRPO stem not from increased exploration or more extensive sampling, but from the fundamental improvements in credit assignment enabled by entropy-gated modulation and implicit process rewards.

For commercial models, we observe that under the 2,048 token generation limit, a fraction of their outputs are truncated before producing a final answer, resulting in lower than expected pass rates on complex reasoning tasks. EP-GRPO achieves correct answers using substantially shorter reasoning paths.

Table~\ref{tab:qwen_pass_k_updated} reports the pass@k ($k \in \{4,8,16\}$) scores for all methods. EP-GRPO achieves the best results across all datasets and model scales. We find that on medium-difficulty datasets such as MATH500 and AMC23, GRPO and its variants exhibit rapid accuracy growth as the number of samples increases, narrowing the gap with EP-GRPO at high $k$ values. However, on high-difficulty datasets such as AIME24 and AIME25, the advantage of EP-GRPO widens substantially with larger $k$ (3B: AIME25 pass@16 38.20\% vs.\ 9.69\%; 7B: AIME25 pass@16 79.84\% vs.\ 58.10\%). This indicates that EP-GRPO's implicit process rewards and progress-aligned normalization enable the model to discover correct reasoning paths that standard GRPO fails to identify, particularly on challenging problems.

\subsection{Ablation Study}

Table~\ref{tab:ablation} reports the ablation results on Qwen2.5-3B, with each variant isolating one or two core components of EP-GRPO. Overall, every component contributes positive gains in average accuracy, but their contributions exhibit notable specialization across datasets.

On datasets with abundant non-zero-variance training steps, such as AMC23 and Minerva, the full EP-GRPO with all three components working in concert achieves the absolute best performance (AMC23: 39.53\% vs.\ the best ablation 36.09\%; Minerva: 9.06\% vs.\ the best ablation 7.97\%). On AIME24, +IPS achieves 6.04\%, close to the full EP-GRPO's 6.25\%. On AIME25, +EG  and +IPS and ZVD both achieve 2.08\%, within one percentage point of the full EP-GRPO's 2.92\%.

This functional specialization across datasets aligns with the nature of the problems each component addresses. The implicit progress signal primarily resolves polarity misalignment, and its benefits are most pronounced on high difficulty, long chain reasoning datasets such as AIME24 and Minerva. Zero-variance degradation targets gradient collapse, playing a critical role in later training stages and on prompts where all rollouts receive identical rewards. Entropy gating provides magnitude modulation that filters noisy signals at low entropy tokens, complementing both of the above mechanisms. The fact that no single ablated variant dominates across all benchmarks confirms that the three components are non-redundant and collectively necessary for robust performance.

\begin{figure*} 
	\centering
	\caption{\textnormal{Comparison of training trends over all datasets, smoothed with EMA ($\alpha=0.2$).}}
	\label{fig:main_training_curves}
	\begin{subfigure}{\textwidth}
		\centering
		\includegraphics[width=0.95\linewidth]{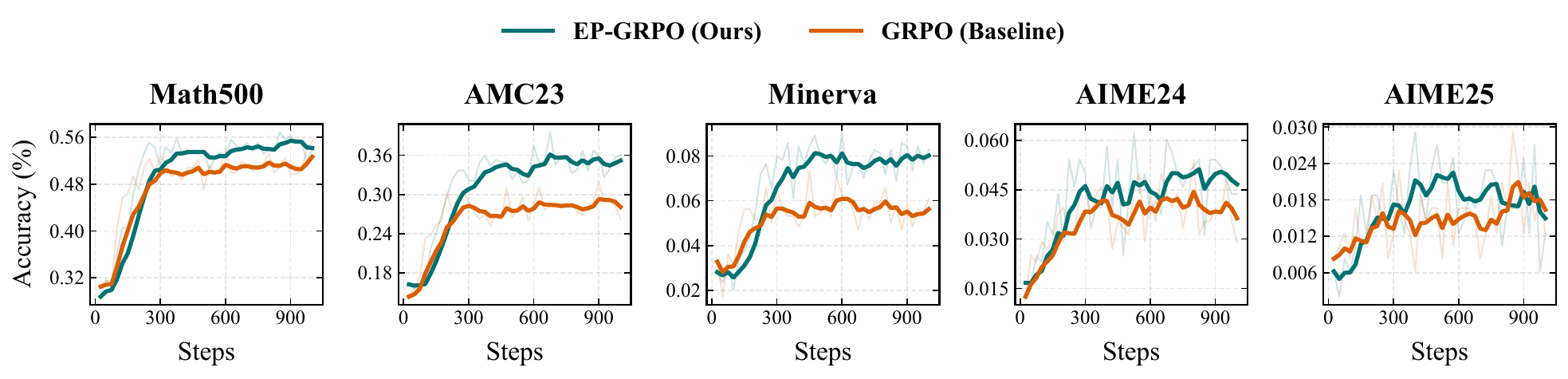}
		\caption{Training dynamics of 3B models.}
		\label{fig:curves_3b}
	\end{subfigure}

	\begin{subfigure}{\textwidth}
		\centering
		\includegraphics[width=0.95\linewidth]{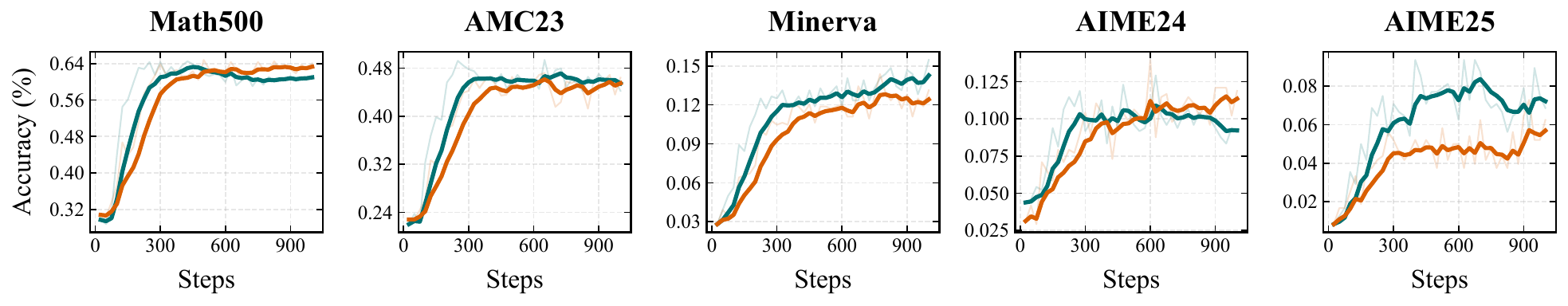}
		\caption{Training dynamics of 7B models.}
		\label{fig:curves_7b}
	\end{subfigure}
	
\end{figure*}

\subsection{Training Dynamics}

Figure~\ref{fig:main_training_curves} visualizes the accuracy trends of GRPO and EP-GRPO over training steps at both model scales. At the 7B scale, EP-GRPO exhibits a stable advantage over GRPO from the earliest training steps. At the 3B scale, EP-GRPO begins to consistently outperform GRPO after approximately 200 steps and progressively widens the gap.

This discrepancy may be attributed to the initial reasoning capability of the base model. The 7B model already possesses a certain level of mathematical reasoning ability at initialization, allowing EP-GRPO's implicit process rewards to provide effective token-level directional feedback from the early stages. The 3B model, by contrast, requires a warm-up period of approximately 200 steps to establish reliable entropy estimates and implicit signals before these mechanisms can take full effect. This interpretation is corroborated by the reward curves in Figure~\ref{fig:tensorboard}, where the reward advantage of EP-GRPO emerges and widens following the same temporal pattern.

\subsection{Training Monitoring Metrics}
\begin{table*}[t]
	\centering
	\small 
	\rmfamily
	\caption{\textnormal{Ablation study on Qwen2.5-3B.  Acc denotes sample accuracy and Fmt denotes format correctness (boxed rate). Best results are in bold.}}
	\label{tab:ablation}
	\renewcommand{\arraystretch}{1.2}
	\setlength{\tabcolsep}{3.5pt}
	
	\begin{tabular}{l cc cc cc cc cc | cc}
		\toprule
		& \multicolumn{2}{c}{\textbf{Math 500}}
		& \multicolumn{2}{c}{\textbf{AMC 23}} 
		& \multicolumn{2}{c}{\textbf{Minerva}} 
		& \multicolumn{2}{c}{\textbf{AIME 24}} 
		& \multicolumn{2}{c}{\textbf{AIME 25}} 
		& \multicolumn{2}{c}{\textbf{Avg}} \\
		\cmidrule(lr){2-3} \cmidrule(lr){4-5} \cmidrule(lr){6-7} \cmidrule(lr){8-9} \cmidrule(lr){10-11} \cmidrule(lr){12-13}
		\textbf{Model} 
		& Acc & Fmt
		& Acc & Fmt
		& Acc & Fmt
		& Acc & Fmt
		& Acc & Fmt
		& Acc & Fmt \\
		\midrule
		
		Base
		& 31.56 & 73.12
		& 15.94 & 74.38
		& 3.13 & 65.94
		& 2.08 & 79.38
		& 0.83 & 80.00
		& 10.71 & 74.56 \\
		
		GRPO
		& 50.94 & 98.28
		& 28.28 & 96.56
		& 6.25 & 99.22
		& 4.58 & 91.88
		& 0.63 & 95.21
		& 18.14 & 96.23 \\
		
		\midrule
		
		\textit{Single Component} \\
		
		+EG
		& 52.66 & 98.59
		& 32.03 & 96.56
		& 6.72 & 98.12
		& 4.58 & 92.29
		& 2.08 & 95.42
		& 19.61 \textbf{\textit{(-14.5\%)}} & 96.20 \\
		
		+IPS
		& 54.06 & 98.75
		& 31.72 & 96.09
		& 6.88 & 96.72
		& 6.04 & 90.83
		& 1.25 & 95.62
		& 19.99 \textbf{\textit{(-12.8\%)}} & 95.60 \\
		
		\midrule
		
		\textit{Two Components} \\
		
		+EG+IPS
		& 52.81 & 98.44
		& 36.09 & 97.19
		& 7.97 & 98.44
		& 4.79 & 92.71
		& 0.83 & 97.50
		& 20.50 \textbf{\textit{(-10.6\%)}}& 96.86 \\
		
		+IPS+ZVD
		& 54.84 & 99.53
		& 35.63 & 95.62
		& 7.97 & 96.56
		& 4.17 & 88.75
		& 2.08 & 96.67
		& 20.94 \textbf{\textit{(-9.5\%)}} & 95.43 \\
		
		\midrule
		
		\textit{Full Model} \\
		
		EP-GRPO (All)
		& \textbf{56.88} & 99.06
		& \textbf{39.53} & 95.31
		& \textbf{9.06} & 96.88
		& \textbf{6.25} & 90.00
		& \textbf{2.92} & 97.50
		& \textbf{22.93} & 95.75 \\
		
		\bottomrule
	\end{tabular}
\end{table*}

Figure~\ref{fig:tensorboard} further presents the trends of key monitoring metrics throughout training. EP-GRPO achieves higher average rewards than GRPO at both model scales, with the reward advantage progressively widening during the middle and late stages of training.

In terms of KL divergence, EP-GRPO maintains slightly higher values than GRPO with occasional spikes, particularly in the early stages of training. This behavior is consistent with the design of the implicit progress advantage: by providing token-level directional feedback derived from policy divergence $s_{i,t}$, EP-GRPO actively encourages the policy to deviate from the reference model at high entropy decision points where the implicit signal is most informative. The occasional KL spikes coincide with training steps where the model encounters unusual or particularly challenging reasoning patterns, triggering stronger token-level corrections. Importantly, despite these transient increases, the KL divergence does not exhibit unbounded growth, confirming that the entropy-gated modulation effectively constrains the regularization to critical tokens rather than allowing uncontrolled deviation.

Regarding gradient norms, EP-GRPO exhibits larger values than GRPO with occasional sharp spikes. This is a direct consequence of the zero-variance degradation mechanism: on training steps where $\text{std}(\mathbf{r}) = 0$, standard GRPO produces near-zero gradient norms because the outcome advantage vanishes, whereas EP-GRPO sustains non-zero gradients through the implicit process advantage $\hat{A}_{i,t}^{\text{progress}}$. These sustained gradients manifest as larger overall gradient norms and, when accumulated across multiple consecutive zero-variance steps, can produce the observed spikes. Rather than indicating training instability, these larger gradient norms reflect EP-GRPO's ability to continue learning from batches that would otherwise be wasted.
\begin{figure*} [t]
	\centering
	\caption{Comparison of training dynamics between EP-GRPO and GRPO.}
	\label{fig:tensorboard}
	\begin{subfigure}{\textwidth}
		\centering
		\includegraphics[width=0.95\linewidth]{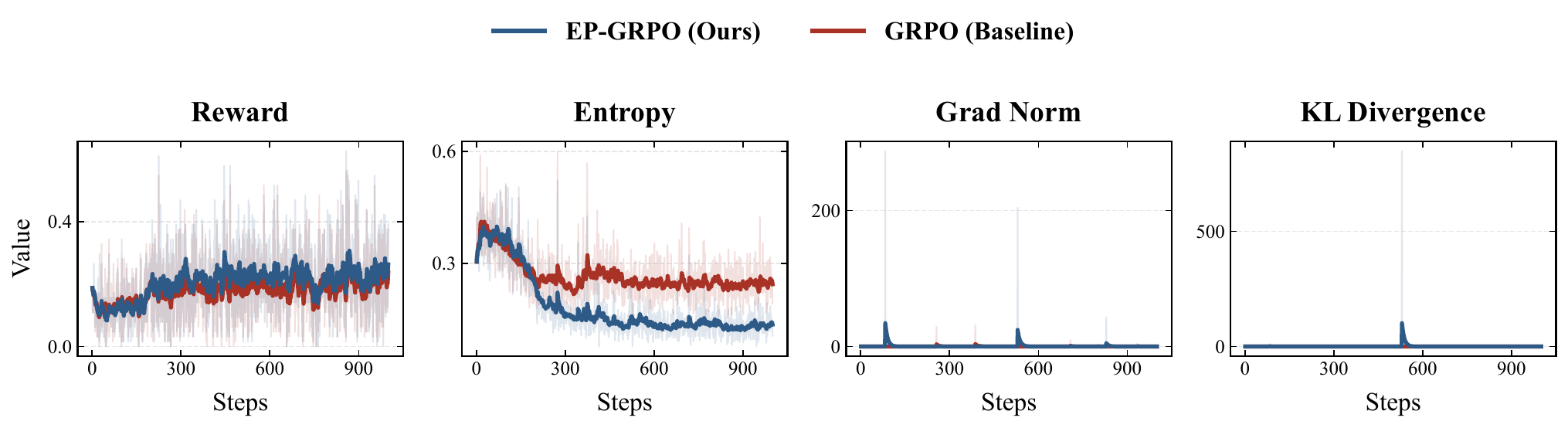}
		\caption{Training dynamics on Qwen2.5-3B.}
		\label{fig:curves_3b}
	\end{subfigure}
	\vspace{0.1em} 
	
	\begin{subfigure}{\textwidth}
		\centering
		\includegraphics[width=0.95\linewidth]{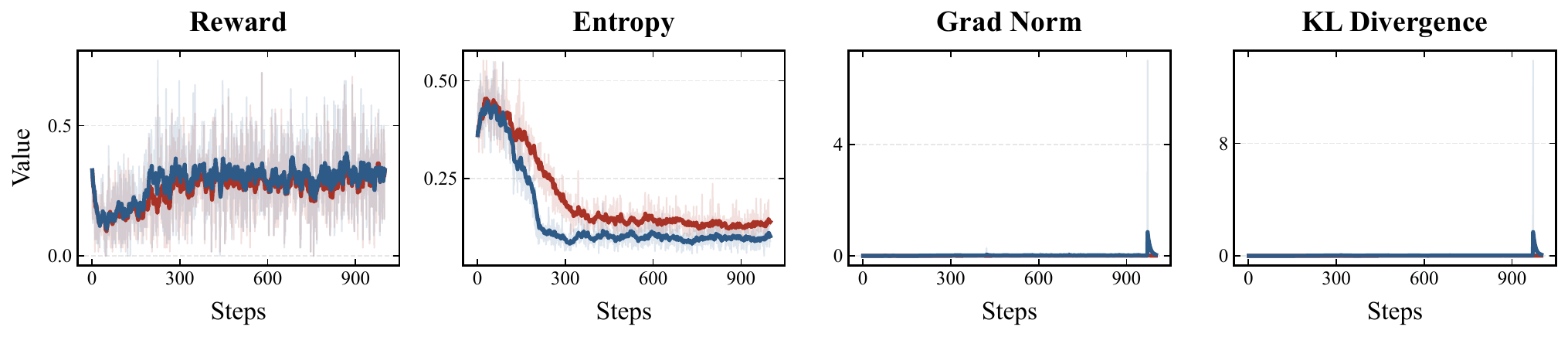}
		\caption{Training dynamics on Qwen2.5-7B.}
		\label{fig:curves_7b}
	\end{subfigure}
\end{figure*}
\begin{figure*}[t]
	\centering
	\includegraphics[width=0.8\textwidth]{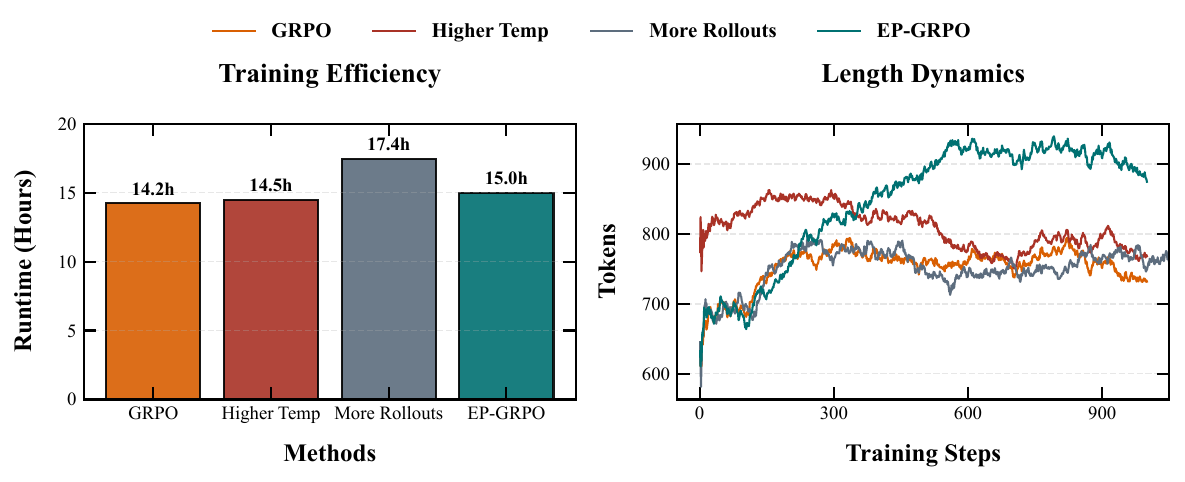}
	\caption{Training efficiency and output length comparison on Qwen2.5-3B.}
	\label{fig:efficiency}
\end{figure*}

\subsection{Computational Efficiency}

Figure~\ref{fig:efficiency} compares the training time and average output length of each method on Qwen2.5-3B. The total training wall-clock time of EP-GRPO is on par with that of standard GRPO. The additional computational overhead introduced by EP-GRPO is negligible, as all required quantities are either already computed during the standard GRPO training pipeline or can be derived with lightweight operations. More Rollouts incurs substantially longer training time due to the increased number of responses generated during the sampling phase.

In terms of output length, EP-GRPO produces reasoning chains that are approximately 20\% longer on average than those of GRPO, indicating that the model engages in deeper and more thorough reasoning. Taken together with the comparable training time, EP-GRPO achieves a superior accuracy efficiency trade-off, it obtains significantly higher accuracy through longer reasoning chains, without incurring additional temporal cost.

\section{Conclusion}
\label{sec:conclusion}

This paper identified three fundamental credit assignment failures in GRPO for RLVR, uniform token granularity, uniform polarity, and zero-variance collapse. We proposed EP-GRPO, a framework that addresses them through entropy-gated modulation, implicit process signals anchored to outcome advantages, and cumulative entropy mapping for progress-aligned normalization. Extensive experiments on two model scales across five mathematical reasoning benchmarks demonstrated consistent and substantial improvements over standard GRPO, with ablation studies confirming the non-redundant contribution of each component and control experiments ruling out alternative explanations. In the future, we will explore more possibilities for process rewards and the laws of entropy-driven exploration.

\bibliographystyle{plain}

\bibliography{cas-refs}

\newpage

\vfill

\end{document}